\title{Hessian of Perplexity for Large Language Models by PyTorch autograd (Open Source)}

\author{Ivan Ilin}
\date{\today}
 
\documentclass[11pt]{article}
\usepackage{hyperref}

\usepackage{amsmath}
\usepackage{amssymb}
\usepackage{mathtools}
\usepackage{amsthm}

\usepackage[capitalize,noabbrev]{cleveref}

\usepackage[colorinlistoftodos,bordercolor=orange,backgroundcolor=orange!20,linecolor=orange,textsize=scriptsize]{todonotes}

\definecolor{mydarkgreen}{RGB}{39,130,67}
\definecolor{mydarkred}{RGB}{192,25,25}

\usepackage{fullpage}

\usepackage{multirow}
\usepackage[cp1250]{inputenc}
\usepackage[T1]{fontenc}
\usepackage{calligra}
\usepackage{amsfonts}
\usepackage{layout}
\usepackage{amsmath}
\usepackage{amsthm}
\usepackage{amssymb}

\usepackage{url}
\usepackage{color}
\usepackage{graphicx}
\usepackage{algorithmic}
\usepackage{algorithm}
\usepackage{verbatim}
\usepackage{natbib}

\usepackage{hyperref}

\usepackage{thm-restate}
\usepackage{nicefrac}

\usepackage{graphicx}
\usepackage{apptools}
\usepackage[flushleft]{threeparttable}
\usepackage{array,booktabs,makecell}
\usepackage{multirow}
\usepackage{nicefrac}
\usepackage{amsthm}
\usepackage{amsmath,amsfonts,bm}
\usepackage{wrapfig}
\usepackage{caption}
\usepackage{subcaption}
\usepackage{siunitx}
\usepackage{thm-restate}
\usepackage{nccmath}
\usepackage{empheq}
\usepackage{bbm}
\usepackage{tabularx}

\usepackage{algorithmic}
\usepackage{algorithm}
\usepackage{filecontents}

\usepackage{color}
\usepackage{colortbl}
\definecolor{bgcolor}{rgb}{0.76,0.88,0.50}
\definecolor{bgcolor0}{rgb}{0.93,0.99,1}
\definecolor{bgcolor1}{rgb}{0.8,1,1}
\definecolor{bgcolor2}{rgb}{0.8,1,0.8}
\definecolor{bgcolor3}{rgb}{0.50,0.90,0.50}
\usepackage{tcolorbox}
\usepackage{pifont}
\definecolor{mydarkgreen}{rgb}{39,130,67}
\definecolor{mydarkred}{rgb}{192,25,25}

\newcommand{\R}{\mathbb{R}} 
\newcommand{\N}{\mathbb{N}} 

\theoremstyle{plain}
\newtheorem{theorem}{Theorem}[section]

\newtheorem{corollary}[theorem]{Corollary}
\theoremstyle{definition}

\theoremstyle{remark}

\newcommand{\eqdef}{:=}

\makeatletter
\newcommand{\vast}{\bBigg@{4}}

\newcommand{\del}[1]{}

\theoremstyle{plain}

\theoremstyle{definition}

\usepackage{natbib}

\allowdisplaybreaks

\begin{document}
\maketitle

\begin{abstract} Computing the full Hessian matrix -- the matrix of second-order derivatives for an entire Large Language Model (LLM) is infeasible due to its sheer size. In this technical report, we aim to provide a comprehensive guide on how to accurately compute at least a small portion of the Hessian for LLMs using PyTorch autograd library. We also demonstrate how to compute the full diagonal of the Hessian matrix using multiple samples of vector-Hessian Products (HVPs). We hope that both this guide and the accompanying GitHub code will be valuable resources for practitioners and researchers interested in better understanding the behavior and structure of the Hessian in LLMs. \href{https://github.com/vectozavr/llm-hessian}{\color{purple}{https://github.com/vectozavr/llm-hessian}}
\end{abstract}

\tableofcontents
\newpage

\section{Introduction}

\subsection{Motivation}
\label{sec:motivation}

Second-order information, encapsulated in the Hessian matrix, is a valuable tool for understanding the optimization landscape of Large Language Models (LLMs). However, computing the full Hessian exactly is infeasible due to memory and hardware constraints. As a result, practical optimization algorithms often rely on approximations, such as Quasi-Newton methods \citep{dennis1977quasi, broyden1967quasi} and momentum-based approaches \citep{kingma2014adam, liu2020improved}, which estimate the Hessian rather than computing it explicitly. Despite these limitations, access to the exact Hessian can be crucial for scientific inquiry, offering deeper insights into the structure and behavior of LLMs.
For instance, this report is part of a broader study on quantization \citep{malinovskii2024pushing}, where the Hessian matrix was analyzed to justify the assumption of its diagonal structure. Our empirical observations confirmed that the Hessian is approximately diagonal, reinforcing the validity of this assumption. In this technical report, we provide a comprehensive guide to computing the exact Hessian for a small subset of LLM parameters, offering insights into its practical computation and applications.

\subsection{Goals of this technical report}
\label{sec:contribution}

We now summarize the key goals of the current report.

\begin{enumerate}
\item Provide a comprehensive guide on how to compute the Hessian matrix by the use of automatic differentiation (AD) and the function \texttt{\color{blue}torch.autograd.functional.hessian} implemented in PyTorch \citep{paszke2019pytorch}.
\item Provide a method to estimate the diagonal elements of the Hessian using PyTorch Vector-Hessian Product function \texttt{\color{blue}torch.autograd.functional.vhp} and Hutchinson's trick \citep{hutchinson1989stochastic, bekas2007estimator}.
\item Provide the efficient implementation of the Hessian computation for different subsets of parameters. Code is available at \href{https://github.com/vectozavr/llm-hessian}{\color{purple}{https://github.com/vectozavr/llm-hessian}}.
\end{enumerate}

\subsection{Notation} All key notation used in this paper is summarized in a tabular form in \Cref{sec:notation}; see Table~\ref{tab:notation_table}.

\section{Background and Related Work: Computing the Hessian Matrix}

The \textbf{Hessian matrix} of a scalar function \(\phi: \mathbb{R}^d \to \mathbb{R}\) is the $d\times d$ matrix of second-order partial derivatives. In coordinates, the Hessian is defined by \(H_{ij}({\bf w}) = \frac{\partial^2 \phi}{\partial [{\bf w}]_i \partial [{\bf w}]_j}({\bf w})\), where $i, j \in \{1, \cdots, d\}$, $[{\bf w}]_i$ is the $i^{\text{th}}$ elements of the vector ${\bf w} \in \R^d$.

For any vector ${\bf x}$, we denote $[{\bf x}]_i$ as the $i^{\text{th}}$ element of the vector ${\bf x}$. Similarly, for any matrix $A$, we denote $[A]_{ij}$ as the element of $A$ in the intersection of the $i^{\text{th}}$ row and $j^{\text{th}}$ column.

\subsection{Finite Difference Approximation (Simple, but not robust)}

One straightforward way to obtain Hessian entries is via \textbf{finite difference methods}, which use function evaluations at perturbed points to approximate derivatives. For example, for two indices $i$ and $j$, a second-order finite-difference formula for the mixed partial is: 

\begin{equation}
H_{ij}({\bf w}) \;\approx\; \frac{\phi({\bf w} + h\,{\bf e}_i + h\,{\bf e}_j)\;-\;\phi({\bf w} + h\,{\bf e}_i)\;-\;\phi({\bf w} + h\,{\bf e}_j)\;+\;\phi({\bf w})}{h^2}\,, \nonumber
\end{equation}
where ${\bf e}_k \in \R^{d}$ is the $k^{\text{th}}$ unit vector, $k \in \{1,\cdots, d\}$ and $h$ is a small step size.

The finite difference approach is simple to implement and can be applied to \emph{any} function as a black box, since it only requires evaluating $\phi({\bf w})$ at various points. No analytical expression or internal knowledge of $\phi$ is needed, which makes this method widely applicable even when $\phi$ is defined by complicated code or is not analytically differentiable. Because of this generality, finite differences have historically been a common choice for numerical differentiation. However, finite difference Hessians can be \textit{numerically unstable} and sensitive to the choice of step size $h$, we demonstrated this fact in our numerical experiments. If $h$ is too large, the approximation incurs significant \textit{truncation error} (missing higher-order terms of the Taylor series); if $h$ is too small, subtractive cancellation and floating-point round-off errors can dominate, causing an unstable result.

Despite these issues, finite differences remain a useful baseline. They can be refined with techniques such as the \emph{complex-step method}.

\subsection{Automatic Differentiation}
\label{sec:auto_diff}

An alternative to numerical approximation is to use \textbf{automatic differentiation (AD)} to compute Hessians exactly (up to machine precision). Automatic differentiation systematically applies the chain rule to the function's computational graph, yielding derivative values without symbolic manipulation. PyTorch \citep{paszke2019pytorch} provides a high-level API that automates this process, returning a matrix of all second derivatives of a given scalar function by using the Python function \texttt{\color{blue}torch.autograd.functional.hessian}.

Automatic differentiation yields \emph{exact} derivatives for the given floating-point function -- there is no truncation error from finite differencing. This means the Hessian computed by autograd is as accurate as the function's numerical precision allows. There is no need to hand-derive formulas or choose step sizes; the process is mechanical and robust. In terms of efficiency, reverse-mode AD is dramatically faster than finite differencing for gradient computation when $d$ is large (since one reverse pass obtains all $d$ partials, instead of $d$ function evaluations). For the Hessian, AD still has an advantage that it can reuse intermediate results and apply vectorized linear algebra, often outperforming naive $O(d^2)$ finite difference loops, especially with high-performance autodiff libraries.

The main drawback of autograd-based Hessians is \textit{computational and memory cost}. Constructing the Hessian still requires significantly more work than a single gradient: in general, computing an $d\times d$ Hessian via reverse-mode AD involves $d$ gradient computations (or the memory to record a computational graph of size proportional to $n$ outputs if using a single forward-mode sweep). The intermediate computational graph for $\nabla\phi$ must be retained to differentiate it again, which can consume a lot of memory for complex functions. Thus, computing a full Hessian for very large $d$ (such as millions of parameters in a deep network) is often infeasible. In practice, one often resorts to Hessian-vector products (which can be obtained by a single backward pass through the gradient, avoiding materializing the whole Hessian) rather than forming $H$ explicitly. Another limitation is that AD requires the function to be implemented in a differentiable manner. Non-differentiable operations or custom code may need special handling (e.g. piecewise-defined functions can be problematic). Frameworks like PyTorch restrict Hessian computation to scalar-output functions.

\subsection{Symbolic Differentiation (SymPy)}

Instead of numerical methods, one can use \textbf{symbolic algebra systems} to derive the Hessian analytically. For a given analytic expression of $\phi({\bf w})$, computer algebra systems (CAS) like SymPy \citep{10.7717/peerj-cs.103} can perform exact differentiation to obtain formulae for each second partial derivative. This yields an exact Hessian matrix (in symbolic form), which can then be evaluated for specific ${\bf w}$ or analyzed for structure. The advantage is that the result is \emph{exact} (no approximation error) and can sometimes be simplified or factored to reveal insights (e.g. showing symmetry or sparsity patterns). Symbolic Hessians are useful in verifying analytical results or generating high-precision reference values. However, the disadvantages are significant for LLMs. Symbolic differentiation tends to suffer from \emph{expression swell} -- intermediate expressions grow exponentially in size, making it impractical for functions of many variables or very complicated forms. In the worst case, the symbolic Hessian expression might be enormous, leading to slow computation and high memory usage. Indeed, direct symbolic differentiation is known to be much slower than automatic differentiation for complex tasks. That is why we decided not to consider this approach in the current report.

\subsection{Estimating the Hessian Diagonal Using Hessian-Vector Products and Randomized Probing}
\label{sec:diag_hes_est}

Computing the full Hessian matrix \( H = \nabla^2 f(x) \) explicitly is often impractical for large-scale models, such as Large Language Models (LLMs), due to memory and computational constraints. However, in many applications, the primary interest lies in the \textit{diagonal} of the Hessian, which provides crucial information about the local curvature of the function \( f(x) \) along individual parameters. A computationally efficient approach to estimating the Hessian diagonal is based on Hessian-vector products (HVPs) combined with randomized probing techniques, such as \textbf{Hutchinson's trick} \citep{hutchinson1989stochastic, bekas2007estimator}.

Hutchinson's trick is a stochastic approach to approximating the trace and diagonal elements of a matrix without explicitly constructing it. The key idea is that for any symmetric matrix \( H \), we can estimate its diagonal entries using randomly sampled probe vectors. Specifically, for a zero-mean random vector \( v \) with independent components, we have:

\begin{equation}
\mathbb{E}[v \odot (H v)] = H_{diag}, \nonumber
\end{equation}

where \( \odot \) denotes element-wise multiplication, and the expectation is taken over the random vector \( v \). This follows from the identity:

\begin{equation}
\mathbb{E}[v_i v_j] =
\begin{cases}
1, & \text{if } i = j, \\
0, & \text{if } i \neq j.
\end{cases} \nonumber
\end{equation}

Thus, an unbiased estimate of the Hessian diagonal can be obtained by computing the Hessian-vector product \( Hv \) and averaging over multiple random vectors.

Instead of explicitly forming \( H \), we compute Hessian-vector products (HVPs) using automatic differentiation techniques (Section~\ref{sec:auto_diff}). Given a function \( f(x) \), the Hessian-vector product can be computed efficiently via second-order automatic differentiation:

\begin{equation}
H v = \nabla (\nabla f(x)^\top v). \nonumber
\end{equation}

This expression indicates that an HVP can be computed by first obtaining the gradient \( \nabla f(x) \) and then taking the directional derivative of the gradient along \( v \). In modern deep learning frameworks such as PyTorch, this can be efficiently implemented using \texttt{\color{blue}torch.autograd.functional.vhp}.

Using Hutchinson's trick and HVPs, we can estimate the Hessian diagonal as follows:

\begin{enumerate}
    \item Sample \( m \) random probe vectors \( v^{(1)}, v^{(2)}, \dots, v^{(K)} \) from a suitable distribution (typically Rademacher or Gaussian)
    \item Compute Hessian-vector products \( H v^{(k)} \) using automatic differentiation.
    \item Compute the element-wise product \( v^{(k)} \odot (H v^{(k)}) \).
    \item Estimate the diagonal by averaging over samples: 
        \begin{equation}
            H_{diag} \approx \frac{1}{K} \sum_{k=1}^{K} v^{(k)} \odot (H v^{(k)}). \nonumber
        \end{equation}
\end{enumerate}

The choice of distribution for \( v \) affects the variance of the estimator. The Rademacher distribution (random vectors with \( v_i \in \{-1,1\} \) with equal probability) is often preferred due to its computational efficiency and lower variance compared to Gaussian sampling \citep{bekas2007estimator}.

This method avoids storing the full Hessian matrix, requiring only vector-sized storage. Each iteration involves only a Hessian-vector product, which is computationally feasible even for large-scale models. The method provides an unbiased estimate of the diagonal, and accuracy improves with more probe vectors.

The estimate depends on the number of random samples; higher accuracy requires a larger, increasing computational cost. This approach estimates only the diagonal, not off-diagonal entries, limiting its usefulness for applications requiring full Hessian information.

\section{Prerequisites and definition of PPL function}

Large language models (LLMs) \citep{zhang2022opt, touvron2023llama, dubey2024llama} are composed of multiple sequential transformation blocks each of which contains several linear layers. To analyze the Hessian we have made various experiments with a OPT-125M model \citep{zhang2022opt}. The OPT-125M model consists of $B=12$ subsequent blocks, each containing multiple linear layers (matrices). For every block $i\in\{1, \cdots, B\}$ we have the following linear layers:
\begin{itemize}
    \item Self attention:
    \begin{itemize}
        \item q\_proj: $Q_i \in \R^{768 \times 768}$
        \item v\_proj: $V_i \in \R^{768 \times 768}$
        \item k\_proj: $K_i \in \R^{768 \times 768}$
        \item out\_proj: $O_i \in \R^{768 \times 768}$
    \end{itemize}
    \item First fully connected layer: $F^1_i \in \R^{768 \times 3072}$
    \item Second fully connected layer: $F^2_i \in \R^{3072 \times 768}$
\end{itemize}

Let us define $k\in \N$ as the number of linear layers in a single block. For OPT-125M, $k=6$. Since we have $B=12$ blocks with $k=6$ linear layers in each, in total there are $L=Bk=72$ linear layers. We denote $W_l \in \R^{d^l_{in} \times d^l_{out}}$ with $l\in \{1,\cdots,L\}$ as one of the linear layer of the model.

Given a layer index $l$, let ${\cal R}_l: \R^{d_{in}^l \times d_{out}^l}  \to \R^{d_{in}^l \cdot d_{out}^l}$ be the ``reshaping'' operator, reshaping a matrix into a large-dimensional vector. That is, ${\bf w}_l ={\cal R}_l(W_l)$ is the vector obtained from the matrix $W_l$ by concatenating entries of $W_l$ into a single $d^l\eqdef d_{in}^l \times d_{out}^l$ dimensional vector. The entries can be concatenated in any order as long as it is always fixed.  Note that $\|W_l\|_F = \|{\cal R}_l(W_l)\|_2 = \|{\bf w}_l\|_2$. Further, let ${\cal R}_l^{-1}$ be the inverse reshaping operator mapping ${\bf w}_l$ back to $W_l$, such that
$ {\cal R}_l^{-1}({\bf w}_l) = {\cal R}_l^{-1}( {\cal R}_l(W_l ))  = W_l$.  Let ${\bf w} \eqdef ({\bf w}_1,\dots,{\bf w}_L ) \in \R^d$, where $d \eqdef \sum_{l=1}^L d^l$, and ${\cal R}^{-1}({\bf w}) \eqdef ({\cal R}_1^{-1}({\bf w}_1), \dots, {\cal R}_L^{-1}({\bf w}_L)) $.

Let $\phi:\R^{d} \to \R$ be the perplexity function on $\R^d$ defined formally as
\begin{equation}
    \label{eq:phi_main_def}
    \phi({\bf w}) \eqdef PPL({\cal R}^{-1}({\bf w})),
\end{equation}
where $PPL$ is the perplexity function operating in the space of $W$.

Computing the full Hessian for even a single matrix from the self attention of the first layer is infeasible due to its size -- $768\times768 = 589,824$ parameters, leading to a Hessian with around $400$ Billion entries. Given these constraints, we focused on a smaller scope of $t\in \N$ parameters of the module from every layer.

\section{Notation for different subsets of parameters}

In order to consider only a subset of parameters of the Hessian matrix, we need to re-parametrize the perplexity function $\phi({\bf w})$ from (\ref{eq:phi_main_def}) to make it a function of only a subset of parameters from ${\bf w}$. In this chapter we formally define such re-parametrizations for different subsets of parameters.

\subsection{Consider a single linear layer from one single block}

Let us consider some layer $W_l \in \R^{d^l_{in} \times d^l_{out}}$. Let us consider the first $t \in \{1, \cdots, d^l_{in} \cdot d^l_{out}\}$ entries of the vector ${\bf {w}}_l$ -- the vector ${\bf {w}}^t_l = [{\bf {w}}_l]_{:t}$, ${\bf {w}}^t_l \in \R^{t}$. For any vector ${\bf x} \in \R^d$ we use the notation $[{\bf x}]_{n_1:n_2} \in \R^{n_2-n_1}$ with $1 \leq n_1 \leq n_2 \leq d$ to represent the vector with elements from the vector ${\bf x}$ with indices $\{n_1, \cdots, n_2\}$. In our notation $[{\bf {w}}_l]_{:t}$ is equivalent to $[{\bf {w}}_l]_{1:t}$. We will use $(, )$ to express the concatenation of several vectors into one single flat vector. For example, when we write $({\bf {w}}^t_l, [{\bf {w}}_l]_{t:})$, we mean that the resulting vector will be the concatenation of the first vector with the second one:
\begin{equation}
    ({\bf {w}}^t_l, [{\bf {w}}_l]_{t:}) \in \R^{d^l_{in} \cdot d^l_{out}}, \nonumber
\end{equation}
where $[{\bf {w}}_l]_{t:} \in \R^{d^l_{in} \cdot d^l_{out} - t}$ -- is the vector of all elements from ${\bf {w}}_l$ except the first $t$ elements.
Finally, we define the perplexity function $\phi_{W_l}: \R^{t} \to \R$ as a function of ${\bf {w}}^t_l$:
\begin{equation*}
    \phi_{W_l}({\bf {w}}^t_l) \eqdef \phi\left( ({\bf {w}}_1, \cdots, ({\bf {w}}^t_l, [{\bf {w}}_l]_{t:}), \cdots, {\bf {w}}_l ) \right).
\end{equation*}

\subsection{All layers from a single block}
If we aim to consider all $k=6$ linear layers from some block $i\in\{1, \cdots, B\}$, then we define ${\bf b}_i \eqdef ({\bf {w}}^t_{ik+1}, \cdots, {\bf {w}}^t_{ik+k})$, where ${\bf {w}}^t_{ik+s} \in \R^t$ is the subset of the first $t$ parameters of the $s^{\text{th}}$ sub-module from the $i^{\text{th}}$ block: ${\bf {w}}^t_{ik+s} = [{\bf w}_{ik+s}]_{:t}$. In this case, the perplexity function $\phi_{b_i}: \R^{kt} \to \R$ can be represented the the following way:
\begin{equation*}
    \phi_{b_i}({\bf b}_i) \eqdef \phi\left( ({\bf {w}}_{1}, \cdots, ([{\bf b}_i]_{:t}, [{\bf {w}}_{ik+1}]_{t:}), \cdots ([{\bf b}_i]_{(k-1)t:kt}, [{\bf {w}}_{ik+k}]_{t:}), \cdots, {\bf {w}}_{L} ) \right). \nonumber
\end{equation*}

\subsection{Consider a single linear layer from all blocks}

Here we aim to consider a single linear layer (for example, the q\_proj from a self attention) for all blocks $i\in\{1,\cdots, B\}$. In this report we consider the q\_proj linear layer from all blocks, but similar experiments might be done for any other modules. We define a subset of relevant parameters in the following way:
\begin{equation}
    {\bf q} \eqdef ([{\cal R}(Q_1)]_{:t}, [{\cal R}(Q_2)]_{:t}, \cdots, [{\cal R}(Q_B)]_{:t}), \nonumber
\end{equation}
In this case, the perplexity function $\phi_Q: \R^{Bt} \to \R$ can be represented the the following way:
\begin{equation*}
    \phi_Q({\bf q}) \eqdef \phi\left( ( (([{\bf q}]_{:t}, [{\cal R}(Q_1)]_{t:}), \cdots ), \cdots, (([{\bf q}]_{(B-1)t:Bt}, [{\cal R}(Q_B)]_{t:}), \cdots )) \right).
\end{equation*}

\subsection{Consider all linear layers from LLM}

Finally, if we aim to consider all layers and all sub-modules, then we define ${\bf {w}}^t \eqdef ({\bf {w}}^t_1, \cdots, {\bf {w}}^t_L)$. The perplexity function $\phi_L : \R^{tL} \to \R$ in this case will be
\begin{equation*}
    \phi_L({\bf {w}}^t) \eqdef \phi\left( ( ([{\bf {w}}^t]_{:t}, [{\bf {w}}_1]_{t:}), \cdots, ([{\bf {w}}^t]_{(L-1)t:Lt}, [{\bf {w}}_L]_{t:} ) ) \right).
\end{equation*}

\section{Experiments}

\subsection{Setup}
In all our experiments we use OPT-125M model \citep{zhang2022opt} with model sequence length of $2048$. If otherwise is not specified, we used a batch size $b=140$ that corresponds to the full WikiText-2~\citep{wikitext103} validation set. For all our experiments we used 4 $\times$ Nvidia A100 GPUs \citep{choquette2021nvidia}.

\subsection{Different subsets of parameters}

You can see the results of the Hessian computation for different subsets of parameters on the Figure~\ref{fig:hessian_different_subsets}.

Firstly, let us consider $t=768$ parameters from only the first layer's q\_proj -- the linear layer $Q_1 \in \R^{768\times768}$. On (Fig.~\ref{fig:hessian_q_proj_t_768}) we can see $\nabla^2 \phi_{W_1}({\bf {w}}^t_1) \in \R^{768 \times 768}$. In this particular case ${\bf {w}}^t_1 \in \R^{768}$ is the first row of the matrix $Q_1 \in \R^{768\times768}$.

On the next step, we considered all linear layers of the first block. We selected $t=300$. We can see $\nabla^2 \phi_{b_1}({\bf {b}}_1) \in \R^{1800 \times 1800}$ (Fig.~\ref{fig:hessian_all_layers_first_block_t_300}).

Then we expanded the Hessian computation to include parameters from multiple layers -- from each layer, we selected $t=150$ parameters of the matrix $Q_i \in \R^{768\times768}$ and repeated this for $i\in \{1,\cdots, B\}$, yielding $\nabla^2 {\phi}_Q({\bf {q}}) \in \R^{1800 \times 1800}$ (Fig.~\ref{fig:hessian_q_proj_all_blocks_t_150}).

Finally, we expanded the Hessian computation to include parameters from all layers and all sub-modules -- from each layer $W_l$, we selected $t=25$ parameters. We can see $\nabla^2 {\phi}_L({\bf w}^t) \in \R^{1800 \times 1800}$ on the (Fig.~\ref{fig:hessian_all_layers_all_blocks_t_25}).

\begin{figure}[ht]
    \centering
    \begin{subfigure}{0.24\linewidth}
        \includegraphics[width=0.99\linewidth]{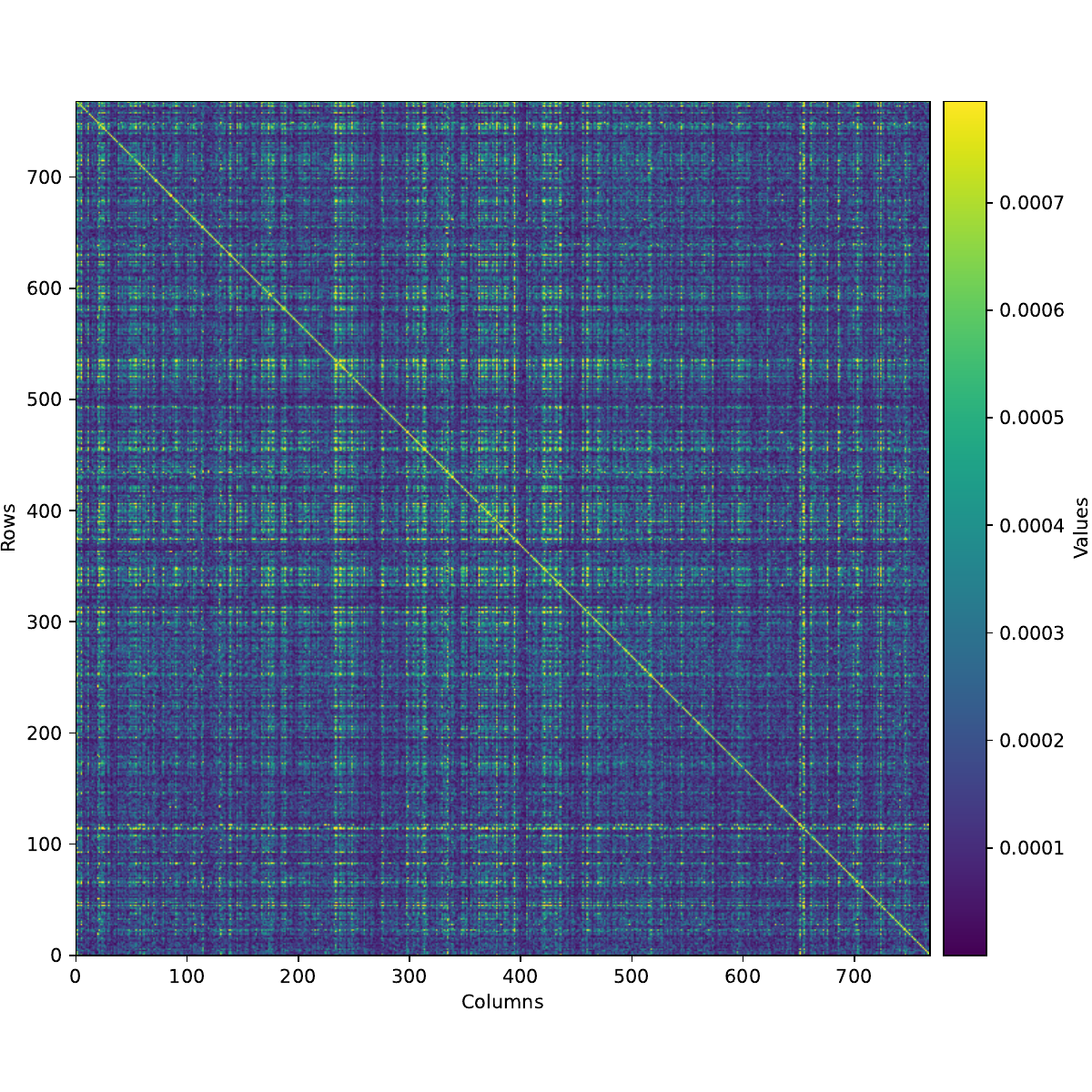}
        \captionsetup{width=.9\linewidth}
        \captionsetup{aboveskip=0pt, belowskip=12pt}
        \caption{$|\nabla^2 \phi_{W_1}({\bf {w}}^t_1)|$, $t=768$, total number of variables = $768$.}
        \label{fig:hessian_q_proj_t_768}
    \end{subfigure}
    \begin{subfigure}{0.24\linewidth}
        \includegraphics[width=0.99\linewidth]{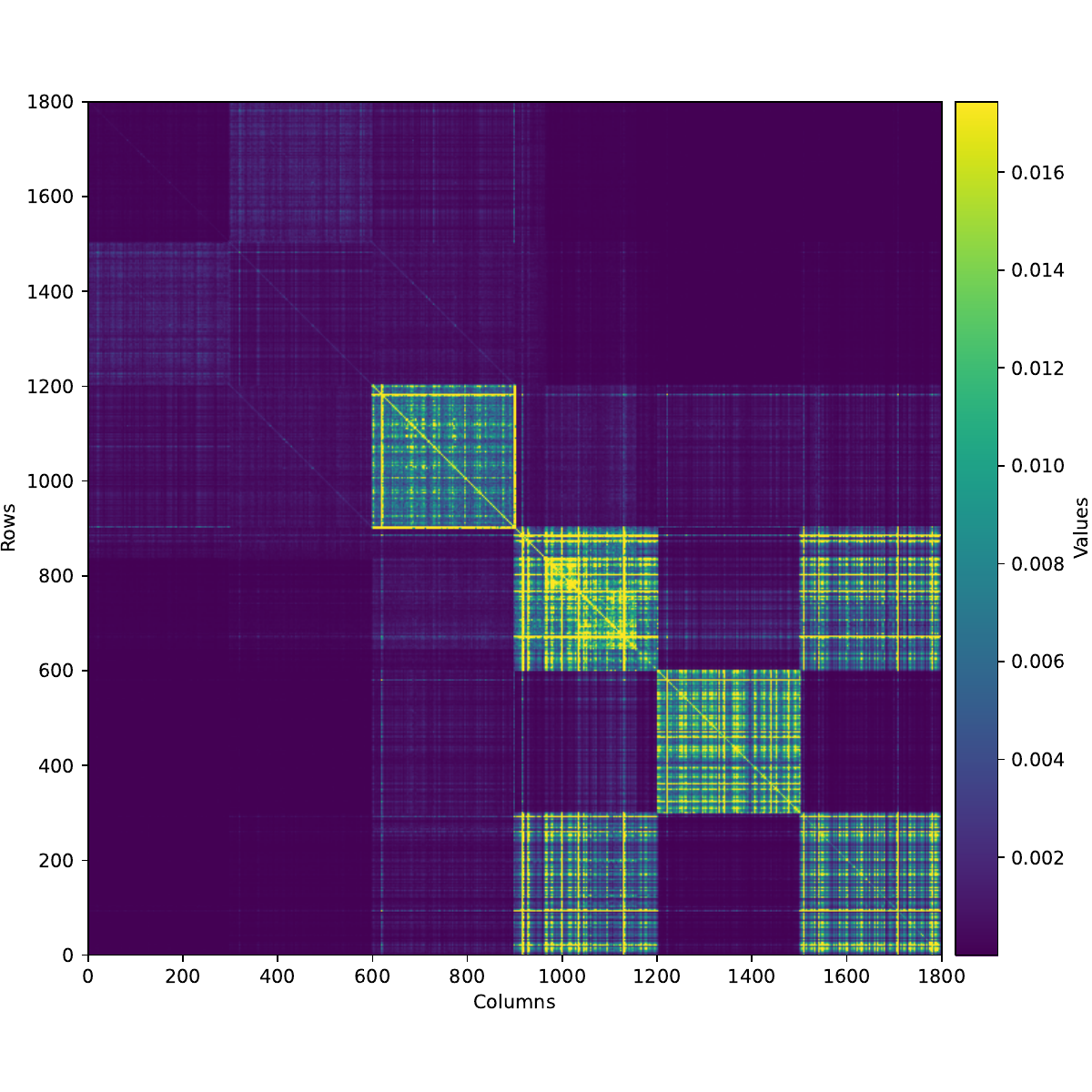}
        \captionsetup{width=.9\linewidth}
        \captionsetup{aboveskip=0pt, belowskip=12pt}
        \caption{$|\nabla^2 \phi_{b_1}({\bf {b}}_1)|$, $t=300$, total number of variables = $1800$.}
        \label{fig:hessian_all_layers_first_block_t_300}
    \end{subfigure}
    \centering
    \begin{subfigure}{0.24\linewidth}
        \includegraphics[width=0.99\linewidth]{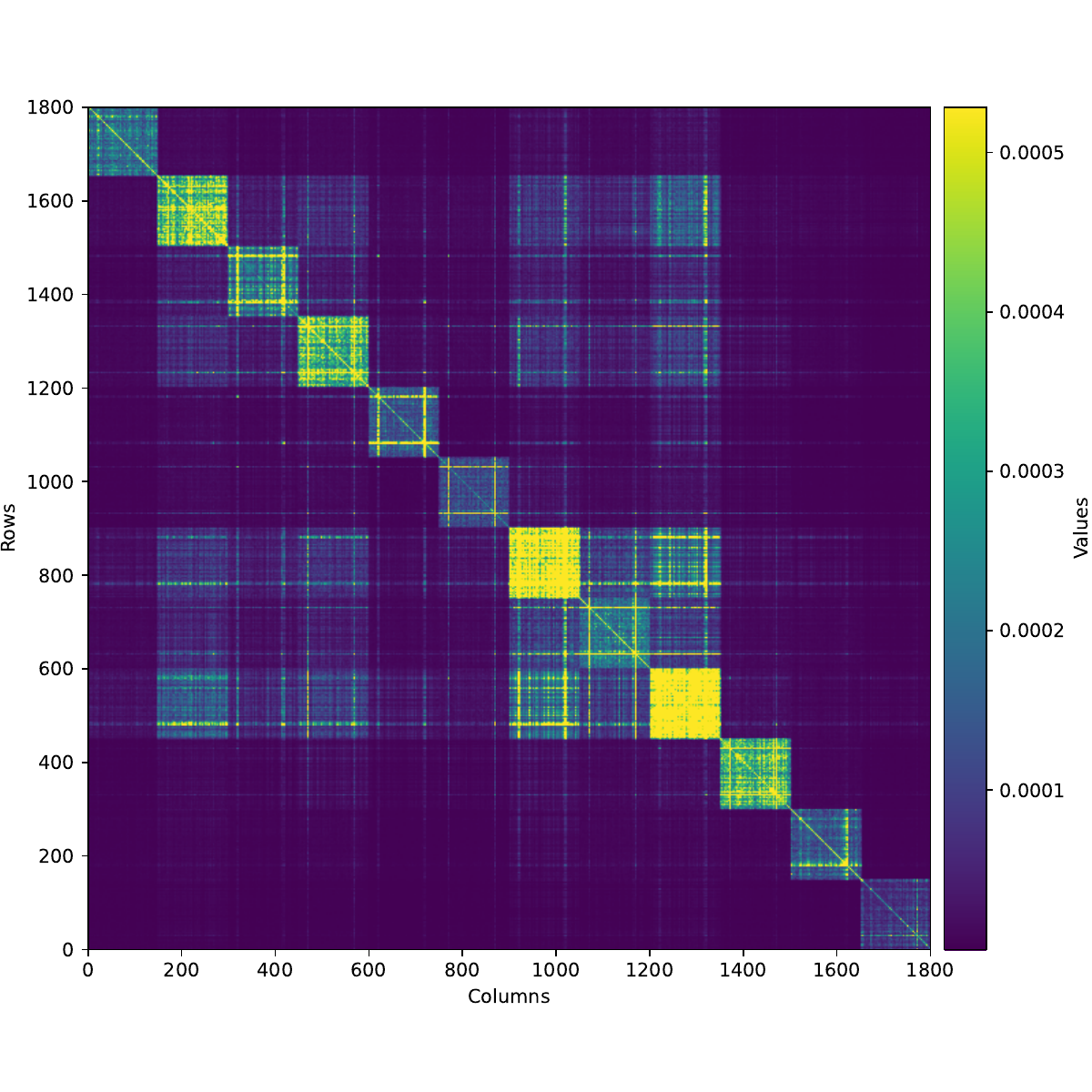}
        \captionsetup{width=.9\linewidth}
        \captionsetup{aboveskip=0pt, belowskip=12pt}
        \caption{$|\nabla^2 {\phi}_Q({\bf {q}})|$, $t=150$, total number of variables = $1800$.}
        \label{fig:hessian_q_proj_all_blocks_t_150}
    \end{subfigure}
    \begin{subfigure}{0.24\linewidth}
        \includegraphics[width=0.99\linewidth]{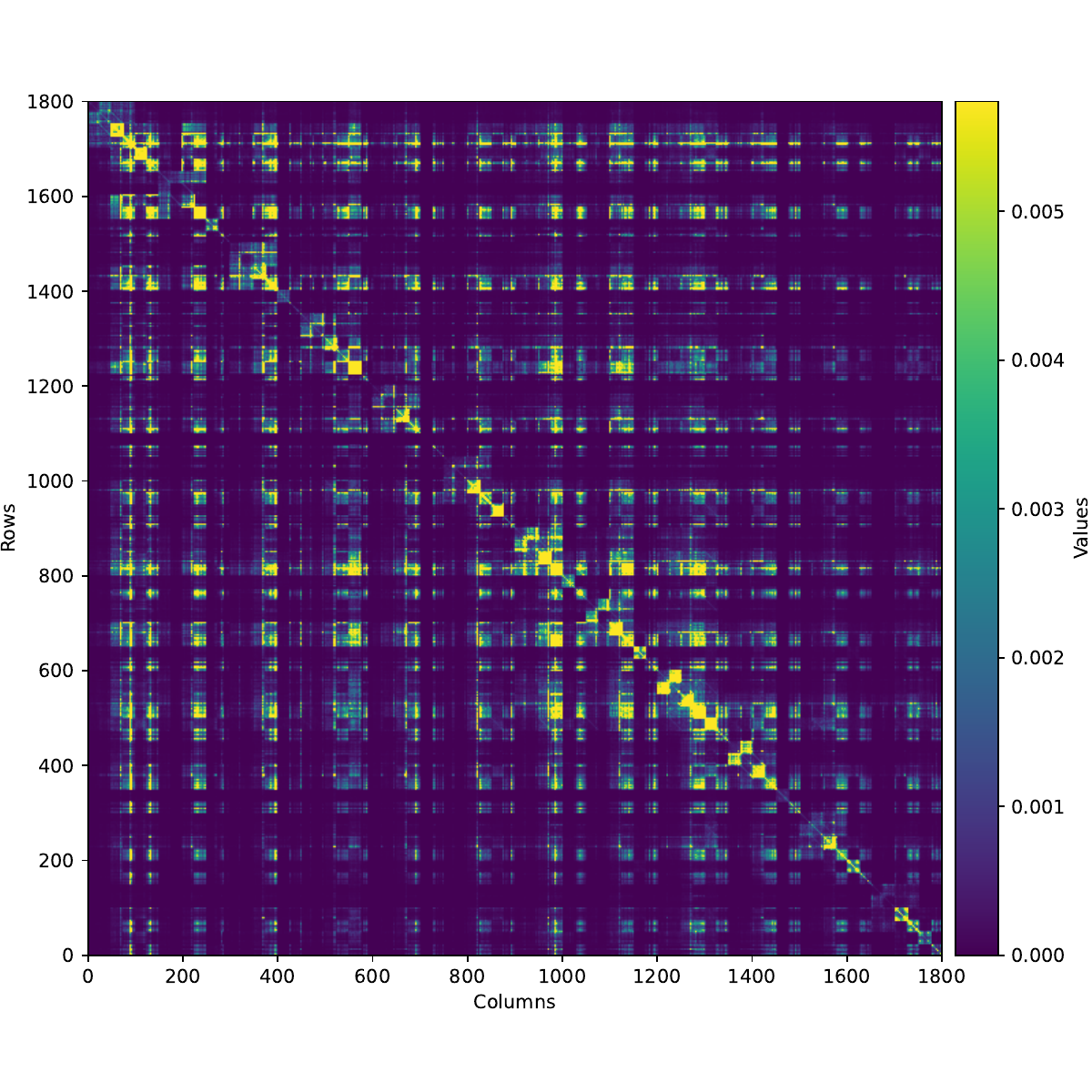}
        \captionsetup{width=.9\linewidth}
        \captionsetup{aboveskip=0pt, belowskip=12pt}
        \caption{$|\nabla^2 {\phi}_L({\bf w}^t)|$, $t=25$, total number of variables = $1800$.}
        \label{fig:hessian_all_layers_all_blocks_t_25}
    \end{subfigure}
    \caption{Visual representation of different parts of $\nabla^2 \phi({\bf w})$ for different subsets of parameters. For clarity, we have plotted the absolute values of the entries in all cases to better visualize the magnitude of the elements.}
    \label{fig:hessian_different_subsets}
\end{figure}

\subsection{Different batch size $b$}

To accurately compute the Hessian matrix, we must evaluate the Hessian over a batch of input tokens and then average the results across different samples. We conducted experiments using various batch sizes $b \in {1, \ldots, 140}$. To assess the accuracy of the Hessian matrix computed with different batch sizes, we plotted the relative $\ell_2$ loss and the relative $\ell_2$ difference between consecutive batch sizes as follows:
\begin{gather} 
\text{Relative $\ell_2$ loss}(b) = \frac{\|H^b - H^{140}\|_2}{\|H^{140}\|_2}, \nonumber \\
\text{Relative $\ell_2$ difference}(b) = \frac{\|H^{b+1} - H^{b}\|_2}{\|H^{b+1}\|_2}, \nonumber 
\end{gather} where $H^b \in \R^{25 \times 25}$ denotes the Hessian matrix computed with batch size $b$, and $H^{140} \in \R^{25 \times 25}$ is the reference Hessian computed using the full test batch.
From Figure~\ref{fig:hessian_different_bs}, we observe that the Hessian can be accurately estimated with relatively small batch sizes, around $b \sim 60$.

\begin{figure}[ht]
    \centering
    \begin{subfigure}{0.4\linewidth}
        \includegraphics[width=1.05\linewidth]{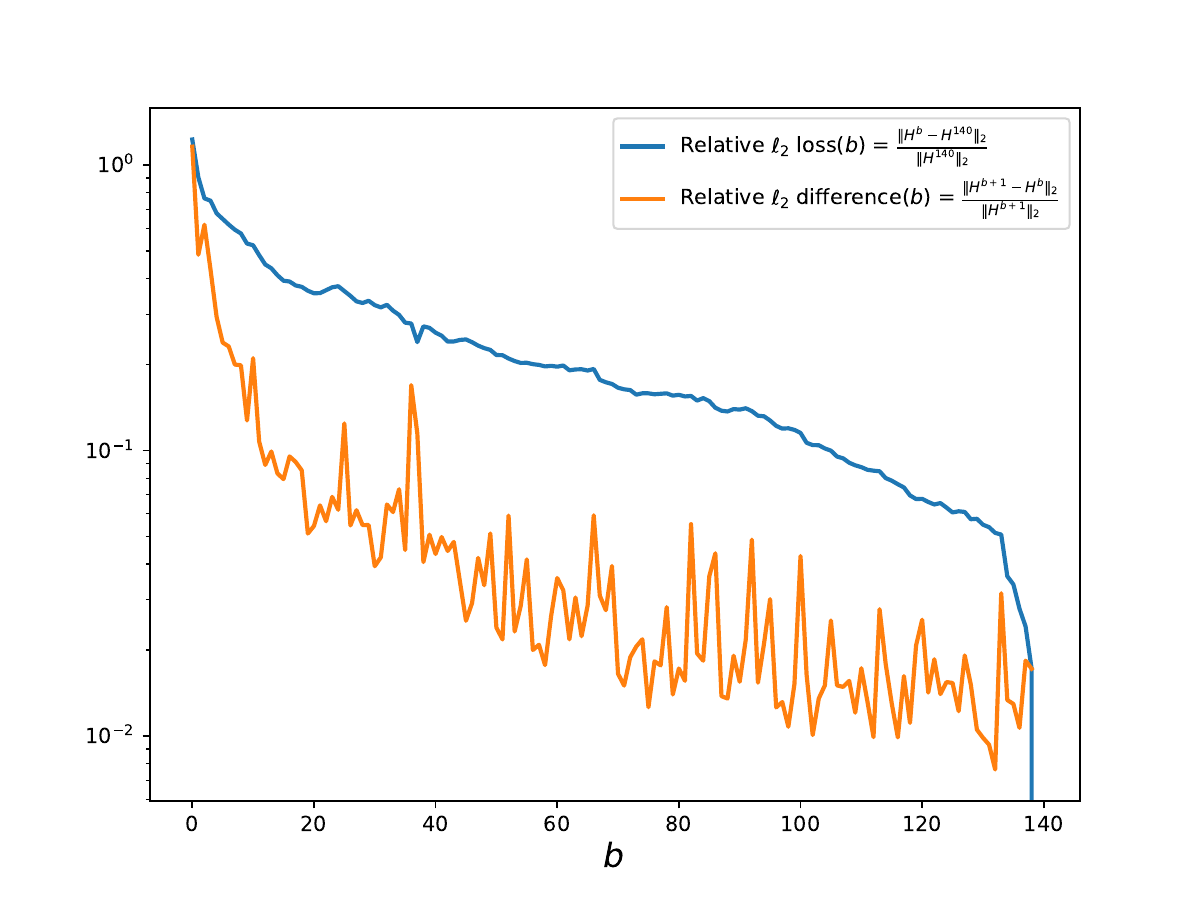}
        \captionsetup{width=.9\linewidth}
        \captionsetup{aboveskip=0pt, belowskip=12pt}
        \caption{Relative $\ell_2$ loss (\ref{eq:l2_relative_loss_hes_diag}) and $\ell_2$ difference (\ref{eq:l2_relative_diff_hes_diag}) for different $b$.}
        \label{fig:losses_vs_bs}
    \end{subfigure}
    \begin{subfigure}{0.59\linewidth}
        \includegraphics[width=1.05\linewidth]{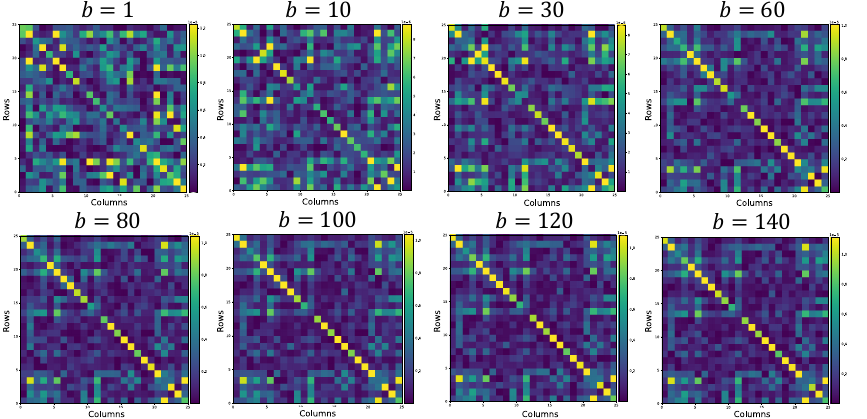}
        \captionsetup{width=.9\linewidth}
        \captionsetup{aboveskip=0pt, belowskip=12pt}
        \caption{$|H^b| = |\nabla^2 \phi_{W_1}({\bf {w}}^t_1)| \in \R^{25 \times 25}$ with different batch size $b$. }
      \label{fig:hessian_different_bs_plots}
    \end{subfigure}
    \caption{Experiments with estimation of the Hessian diagonal elements for q\_proj linear layer from the first block of OPT-125M.}
    \label{fig:hessian_different_bs}
\end{figure}

\subsection{Hessian diagonal estimation}
In this section we aim to estimate the diagonal of the Hessian matrix for all parameters of the linear layer q\_proj from the first block of OPT-125M, $Q_1 \in \R^{768 \times 768}$. The number of diagonal elements from the Hessian in this case will be $768 \cdot 768$. We will denote the diagonal of the Hessian by the tensor $H_{diag} \in \R^{768 \times 768}$ of the same shape as the weight matrix $Q_1 \in \R^{768 \times 768}$.

In order to estimate the tensor $H_{diag}$, we used the Hutchinson's trick (Section~\ref{sec:diag_hes_est}). This approach provides a sequence of tensors \( \widetilde{H}_{diag}^{(1)}, \widetilde{H}_{diag}^{(2)}, \dots, \widetilde{H}_{diag}^{(K)} \), where 
\begin{equation}
    \widetilde{H}^{(k)}_{diag} = \frac{1}{k} \sum_{i=1}^{k} v^{(i)} \odot (H v^{(i)}), \nonumber
\end{equation}
$\widetilde{H}^{(k)}_{diag} \in \R^{768 \times 768}$ for $k=\{1, \cdots, K\}$. 

More iterations $K$ of sampling of vector-Hessian product provide a better approximation of the real Hessian diagonal $H_{diag}$. In order to estimate, how many steps are needed, we computed the true diagonal of the Hessian for the first row of the linear layer ($768$ parameters) by \texttt{\color{blue}torch.autograd.functional.hessian} (Section~\ref{sec:auto_diff}) and computed the partial relative $\ell_2$ loss with the estimation of the same elements from $\widetilde{H}^k_{diag}$:
\begin{equation}
    \label{eq:l2_relative_loss_hes_diag}
    \text{Partial relative $\ell_2$ loss}(k) = \frac{\|[\widetilde{H}^k_{diag} - H_{diag}]_{1}\|_2}{\|[H_{diag}]_{1}\|_2},
\end{equation}
where $[H_{diag}]_{1} \in \R^{1 \times 768}$ is the first row of $H_{diag}$, $[\widetilde{H}^k_{diag} - H_{diag}]_{1} \in \R^{1 \times 768}$ is the first row of the difference between the real diagonal Hessian $H_{diag}$ and its approximation $\widetilde{H}^k_{diag}$.

In addition to the relative loss (\ref{eq:l2_relative_loss_hes_diag}), we also evaluated the relative difference between each two subsequent approximations of the Hessian diagonal:
\begin{equation}
    \label{eq:l2_relative_diff_hes_diag}
    \text{Relative $\ell_2$ difference}(k) = \frac{\|\widetilde{H}^{k+1}_{diag} - \widetilde{H}^{k}_{diag}\|_2}{\|\widetilde{H}^{k+1}_{diag}\|_2}.
\end{equation}

You can see the results of the algorithm for different values of $k$ in Figure~\ref{fig:diagonal_hessian_experiments}. A larger $k$ yields a less noisy and smoother approximation of $H_{\text{diag}}$. To achieve a relative $\ell_2$ loss of $0.25$ (\ref{eq:l2_relative_loss_hes_diag}), we required $3000$ iterations, which took approximately 25 hours on 4 $\times$ NVIDIA A100 GPUs. However, we observed that even a single evaluation of the vector-Hessian product ($k=1$) captures the main patterns of $H_{\text{diag}}$: as shown in Figure~\ref{fig:q_proj_hessian_diag_different_k}, the structure of the Hessian matrix at $k=1$ closely resembles that at $k=5000$, albeit with some additional noise.

\begin{figure}[ht]
    \centering
    \begin{subfigure}{0.48\linewidth}
        \includegraphics[width=1.05\linewidth]{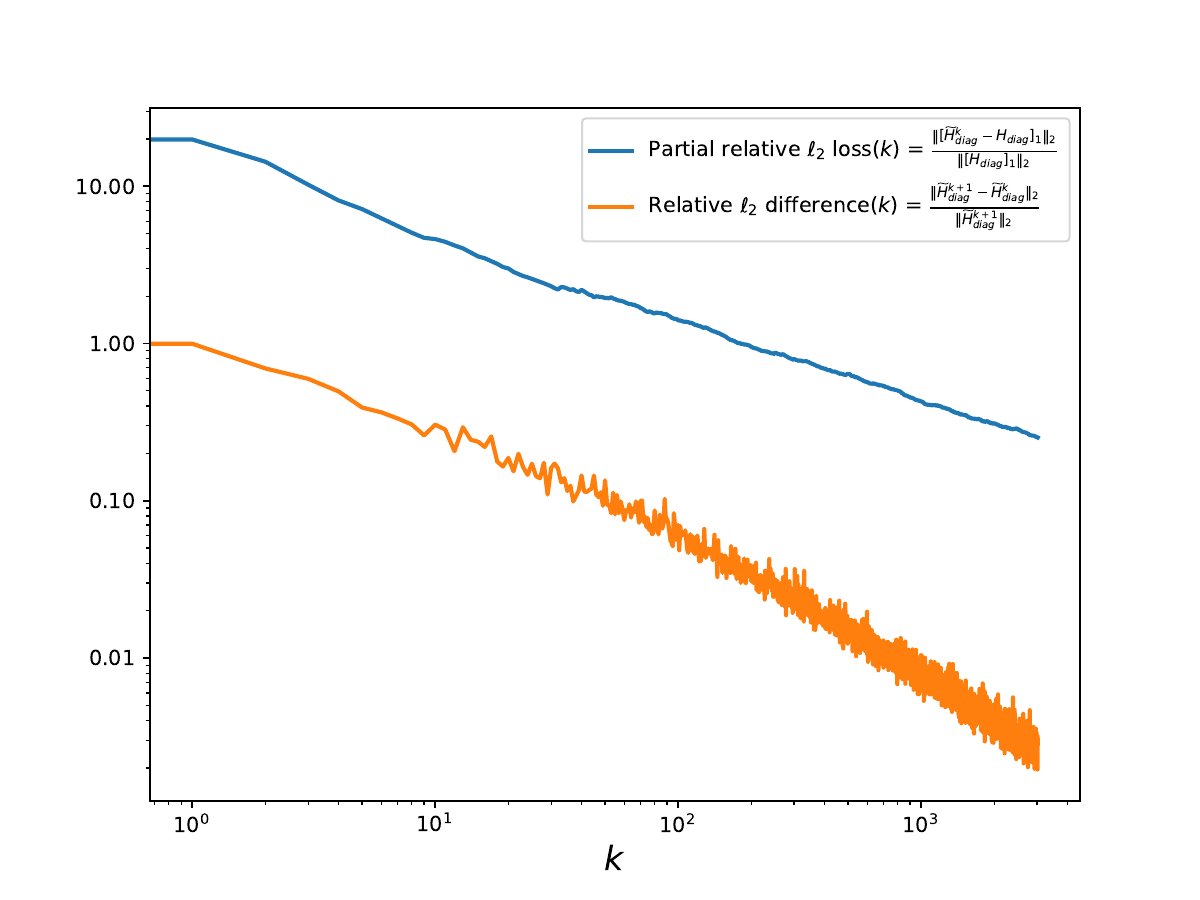}
        \captionsetup{width=.9\linewidth}
        \captionsetup{aboveskip=0pt, belowskip=12pt}
        \caption{Relative $\ell_2$ loss (\ref{eq:l2_relative_loss_hes_diag}) and $\ell_2$ difference (\ref{eq:l2_relative_diff_hes_diag}) for different $k$ (number of VHP samples), $b=60$.}
        \label{fig:losses_vs_k}
    \end{subfigure}
    \begin{subfigure}{0.51\linewidth}
        \includegraphics[width=1.05\linewidth]{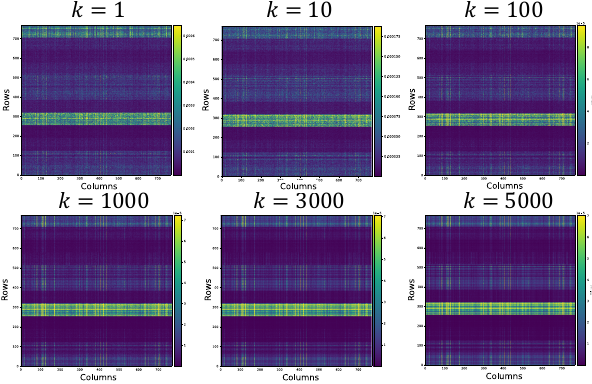}
        \captionsetup{width=.9\linewidth}
        \captionsetup{aboveskip=0pt, belowskip=12pt}
        \caption{$|\widetilde{H}^{(k)}_{diag}| \in \R^{768 \times 768}$ with different $k$ (number of VHP samples), batch size $b=60$. }        \label{fig:q_proj_hessian_diag_different_k}
    \end{subfigure}
    \caption{Experiments with estimation of the Hessian diagonal elements for q\_proj linear layer from the first block of OPT-125M.}
    \label{fig:diagonal_hessian_experiments}
\end{figure}

\section{Implementation details}

To be able to compute the Hessian for only a subset of the network's parameters, we manually defined a perplexity function as a function of the specific subset of parameters. We used this with the PyTorch \citep{paszke2019pytorch} autograd routines to compute the Hessian quickly and accurately.

Note that Hessian computation induces significant memory consumption when using larger batch sizes. A larger batch size is crucial for accurate perplexity computation, and therefore for accurate Hessian computation (for instance, a batch size of $140$ yields a WikiText-2 perplexity for OPT-125M of $27.65$, while a batch size of $4$ results in a WikiText-2 perplexity of $30.06$). To mitigate the memory overflow problem, we modified the perplexity function to exhibit an additive property (Sec.~\ref{sec:hessian_for_large_bs}). This means that we can compute the Hessian for the full batch by averaging Hessians, computed over smaller batches. With this adjustment, we were able to use PyTorch's autograd routine to compute the Hessian on a full batch size without encountering memory overflow issues.

\section{How to compute a Hessian for large batch sizes}
\label{sec:hessian_for_large_bs}

The perplexity function is computed by the following sequence of events:

\begin{enumerate}
    \item We have the input for LLM in a form of the matrix $X \in \R^{b \times s}$, where $b$ is a batch size and $s$ is an output sequence length.
    \item Having the input $X$, we compute the output of the model by the function $f:\R^{b\times s} \to \R^{b\times s\times n}$ that represents the work of LLM, where $n$ is the size of the embedding space. For OPT-125M, $s=2048, n=50272$. Elements of $f(X)$ are called logits and represent the probability for each word to be the next token in the output sequence.
    \item From the output of the function $f$, we compute the Cross Entropy Loss of the output by the function $g:\R^{b\times s\times n} \to \R^{b}$, $$g(f(X)) \eqdef \text{CrossEntropyLoss}(f(X)).$$
    \item After that we compute the average of the Cross Entropy Loss for all elements from a batch: $$\Bar{c} \eqdef \frac{1}{b}\sum_{i=1}^{b}{c_i}.$$
    \item Finally, we compute the Perplexity of the model via: $$P(\Bar{c}) \eqdef e^{\Bar{c}}.$$
\end{enumerate}

Let us change steps (3) and (4): on the step (3) we will not divide the sum by $b$, so we define $$\Bar{c}' = \sum_{i=1}^{b}{c_i},$$ on the step (4) we will not use the exponential function -- instead we will use the identity function: 
\begin{equation}
    \label{eq:new_def_of_ppl}
    P'(\Bar{c}') \eqdef \Bar{c}'.
\end{equation}

\begin{theorem}
\label{th:th_additive_prop}
$PPL'(\Bar{c}')$ defined in (\ref{eq:new_def_of_ppl}) has the additive property. In other words, the perplexity computed for the full batch $b$ will be equal to the sum of perplexities, computed for $b$ subsequent samples.
\end{theorem}
\begin{proof}
Let us define the functions $\hat{f}:\R^{s} \to \R^{s\times n}$ and $\hat{g}:\R^{s\times n} \to \R$ to be the same as $f:\R^{b\times s} \to \R^{b\times s\times n}$ and $g:\R^{b\times s\times n} \to \R^{b}$, but with fixed $b=1$, so we effectively have a reduction of one dimension.

Since the output is computed for each sample from a batch independently, the full batched output $f(X)$ can be obtained by concatenation of $b$ per-sample outputs from $\hat{f}([X]_{i})$ for $i \in \{1, \cdots, b\}$:
$$
f(X) = \left(\hat{f}([X]_{1}), \cdots, \hat{f}([X]_{b})\right),
$$
where $[X]_{i} \in \R^{1 \times s}$, $i \in \{1, \cdots, b\}$ represents the $i^{\text{th}}$ row-vector of the input matrix $X \in \R^{b \times s}$.
The same is true for a Cross Entropy Loss function $g(f(X))$:
$$
g(f(X)) = \left(g(f(X))_1, \cdots, g(f(X))_b\right) = \left(\hat{g}(\hat{f}([X]_{1})), \cdots, \hat{g}(\hat{f}([X]_{b}))\right),
$$
hence the equation (\ref{eq:sum_inside_proof_additive_prop}) holds:
\begin{equation}
    \label{eq:sum_inside_proof_additive_prop}
    P'(X) = \sum_{i=1}^b{\left[g(f(X))\right]_i} = \sum_{i=1}^b{\hat{g}(\hat{f}([X]_{i}))}
\end{equation}
where $X_{i} \in \R^s$.

In essence, equation (\ref{eq:sum_inside_proof_additive_prop}) means that the perplexity computed for the full batch $b$ will be equal to the sum of perplexities, computed for $b$ subsequent samples.

\end{proof}

\begin{corollary}
    We can compute the Hessian of the perplexity function (\ref{eq:new_def_of_ppl}) over large batch of samples by summing up several Hessians, computed on a single sample:
    \begin{equation}
        \label{eq:sum_of_hessians_col}
        \nabla^2 P'(g(f(X))) = \sum_{i=1}^b{\nabla^2 P'(\hat{g}(\hat{f}([X]_{i})))}.
    \end{equation}
\end{corollary}

\section{Table of Frequently Used Notation} \label{sec:notation}

\begin{table}[H]\caption{Notation table}
\label{tab:notation_table}
\centering 
\begin{tabular}{r c p{15cm} }
\toprule
    $[{\bf x}]_i$ & -- & $i^{\text{th}}$ element of the vector ${\bf x}$. \\
    $[W]_{ij}$ & -- & one entry of the matrix $W$ on the intersection of the $i^{\text{th}}$ row and $j^{\text{th}}$ column. \\
    $[W]_{i}$ & -- & $i^{\text{th}}$ row-vector from the matrix $W \in \R^{a \times b}$, $[W]_{i} \in \R^{1 \times b}$. \\
    $k$ & -- & Number of linear layers inside a single block of LLM. For OPT-125M, $k=6$. \\
    $B$ & -- & Number of block of LLM. For OPT-125M, $k=12$. \\
    $L$ & -- & Number of linear layers (weight matrices) of LLM $L=Bk$. For OPT-125M, $L=72$. \\
    $l$ & -- & Index of one linear layer from LLM, $l \in \{1, \cdots, L\}$ \\
    $W_l$ & -- & One linear layer (matrix) of LLM, $W_l \in \R^{d^l_{in} \times d^l_{out}}$. \\
    ${\cal R}_l$ & -- & Reshaping operator, reshaping a matrix into a flat vector, ${\cal R}_l : \R^{d_{in}^l \times d_{out}^l}  \to \R^{d_{in}^l \cdot d_{out}^l}$ \\
    ${\bf w}_l$ & -- & Flat representation of one linear layer (matrix) of LLM, ${\bf w}_l ={\cal R}_l(W_l)$. \\
    $d$ & -- & Total number of parameters of LLM. \\
    $b$ & -- & Batch size for the input of LLM. \\
    $s$ & -- & Output sequence length of the LLM. \\
    $X$ & -- & input matrix for LLM, $X \in \R^{b \times l}$ \\
    $f$ & -- & Mapping that represents the action of LLM on the input $X$, $f:\R^{b\times l} \to \R^{b\times l\times n}$. \\
    $\phi$ & -- & Perplexity function of LLM. \\
    ${\bf e}_i$ & -- & $i^{\text{th}}$ unit vector, ${\bf e}_i \in \R^{d}$, $i \in \{1,\cdots, d\}$ \\
    $\odot$ & -- & Element-wise multiplication of vectors/matrices. \\
\bottomrule
\end{tabular}
\end{table}

\clearpage
\bibliography{references}
\bibliographystyle{plainnat}

\end{document}